\documentclass{article}
\usepackage{ijcai21}

\usepackage{times}
\usepackage{soul}
\usepackage{url}
\usepackage[hidelinks]{hyperref}
\usepackage[utf8]{inputenc}
\usepackage[small]{caption}
\usepackage{graphicx}
\usepackage{amsmath,amssymb}
\usepackage{amsthm}
\usepackage{booktabs}
\usepackage{adjustbox}
\usepackage{multirow,array}
\newcommand\citet[1]{\citeauthor{#1}~[\citeyear{#1}]}

\usepackage[dvipsnames]{xcolor}
\usepackage[ruled,linesnumbered]{algorithm2e}
\usepackage{xspace} 
\usepackage{paralist}
\setdefaultenum{\bfseries(i)}{}{}{}
\usepackage[11pt]{moresize}
\usepackage[normalem]{ulem}

\newcommand\m[1]{\ensuremath{\mathcal{#1}}}
\newcommand\ignore[1]{}
\usepackage[capitalise]{cleveref}

\newcommand\voc{\ensuremath{\Sigma}\xspace}
\newtheorem{thm}{Theorem}
\newtheorem{definition}[thm]{Definition}

\newtheorem{proposition}[thm]{Proposition} 
\newtheorem{theorem}[thm]{Theorem}
\newtheorem{ex}{Example}
\newtheorem{example}[ex]{Example}

\newcommand\muses[1]{\ensuremath{\mathit{MUSs}(#1)}\xspace}
\newcommand\mcses[1]{\ensuremath{\mathit{MCSs}(#1)}\xspace}

\newcommand\setstohit{\ensuremath{\m{H}}\xspace}
\newcommand\F{\ensuremath{\m{F} }\xspace}

\newcommand\Iend{\ensuremath{I_\mathit{end} }\xspace}
\newcommand\formula{\ensuremath{\m{F} }\xspace}
\newcommand\formulac{\ensuremath{\m{C} }\xspace}

\newcommand\mm[1]{\ensuremath{#1}\xspace}
\newcommand\nat{\mm{\mathbb{N}}}

\newcommand\satsets{\mm{\mathbf{SSs}}}

\newcommand\call[1]{\mm{\textsc{#1}}}

\newcommand\cohs{\mm{\call{CondOptHittingSet}}}

\newcommand\sat{\mm{\call{sat}}}
\newcommand\grow{\mm{\call{Grow}}}
\newcommand\omus{\mm{\call{OUS}}}
\newcommand\comus{\mm{\call{OCUS}}}

\newcommand\hitsetbased{hitting set--based\xspace} 

\SetKwInOut{Input}{Input}
\SetKwInOut{OptInput}{Optional}
\SetKwInOut{Output}{Output}
\SetKwInOut{State}{State}
\SetKwInOut{Ext.}{Ext}
\SetKwComment{command}{/*}{*/}

\newcommand\negset[1]{\mm{\overline{#1}}}
\newcommand\maxsat{MaxSAT\xspace}

\title{Efficiently Explaining CSPs with Unsatisfiable Subset Optimization}

\author{
	Emilio Gamba$^1$ \and Bart Bogaerts$^1$ \And Tias Guns$^{1,2}$
	\affiliations
	$^1$Vrije Universiteit Brussel, Belgium\\
	$^2$KU Leuven, Belgium\\
	\emails
	emilio.gamba@vub.be, bart.bogaerts@vub.be,
	tias.guns@kuleuven.be
}

\usepackage{etoolbox}

\newcommand\setcitation[2]{%
  \csdef{mycommoncitation#1}{#2}}
\newcommand\getcitation[1]{%
  \csuse{mycommoncitation#1}}

\setcitation{IDP}{WarrenBook/DeCatBBD14}
\setcitation{idp}{WarrenBook/DeCatBBD14}
\setcitation{fodot}{tocl/DeneckerT08}
\setcitation{foid}{tocl/DeneckerT08}
\setcitation{FOID}{tocl/DeneckerT08}
\setcitation{cplogic}{journal/tplp/VennekensDB10}
\setcitation{CPlogic}{journal/tplp/VennekensDB10}
\setcitation{CPLogic}{journal/tplp/VennekensDB10}
\setcitation{CP}{fai/Rossi06}
\setcitation{cp}{fai/Rossi06}
\setcitation{EZCSP}{lpnmr/Balduccini11}
\setcitation{KR}{Baral:2003}
\setcitation{ASPComp2}{lpnmr/DeneckerVBGT09}
\setcitation{ASPComp3}{journals/tplp/CalimeriIR14}
\setcitation{ASPComp4}{conf/lpnmr/AlvianoCCDDIKKOPPRRSSSWX13}
\setcitation{ASPComp5}{journals/ai/CalimeriGMR16}
\setcitation{ASPComp6}{jair/GebserMR17}
\setcitation{ASPComp7}{tplp/GebserMR20}
\setcitation{CPSupport}{ictai/DeCat13}
\setcitation{CPsupport}{ictai/DeCat13}
\setcitation{functionDetection}{iclp/DeCatB13}
\setcitation{FunctionDetection}{iclp/DeCatB13}
\setcitation{fodot2asp}{DeneckerLTV12} 
\setcitation{Tarskian}{DeneckerLTV12} 
\setcitation{TarskianSemanticsASP}{DeneckerLTV12} 
\setcitation{Inca}{iclp/DrescherW12}
\setcitation{csp2asp}{ijcai/DrescherW11}
\setcitation{DPLLT}{cav/GanzingerHNOT04}
\setcitation{AspInPractice}{synthesis/2012Gebser}
\setcitation{ASPInPractice}{synthesis/2012Gebser}
\setcitation{clasp}{ai/GebserKS12}
\setcitation{oclingo}{kr/GebserGKOSS12}
\setcitation{clingo}{iclp/GebserKKOSW16}
\setcitation{gringo}{lpnmr/GebserST07}
\setcitation{cmodels}{aaai/GiunchigliaLM04}
\setcitation{inputster}{tplp/Jansen13}
\setcitation{DLV}{tocl/LeonePFEGPS06}
\setcitation{LearningPaper}{TPLP/BruynoogheBBDDJLRDV} 
\setcitation{clog}{iclp/BogaertsVDV14} 
\setcitation{foc}{iclp/BogaertsVDV14}
\setcitation{FOC}{iclp/BogaertsVDV14}
\setcitation{inferenceClog}{ecai/BogaertsVDV14}
\setcitation{examplesClog}{nmr/BogaertsVDV14b} 
\setcitation{AFT}{DeneckerMT00}
\setcitation{KBS}{iclp/DeneckerV08}
\setcitation{KBS-invitedtalk}{jelia/Denecker16}
\setcitation{KBPE}{inap/DePooterWD11}
\setcitation{lazyGrounding}{jair/CatDBS15} 
\setcitation{LazyGrounding}{jair/CatDBS15} 
\setcitation{lazygrounding}{jair/CatDBS15} 
\setcitation{lazygroundingASP}{ijcai/BogaertsW18} 
\setcitation{justifications}{lpnmr/DeneckerBS15} 
\setcitation{justificationsAlpha}{ijcai/BogaertsW18} 
\setcitation{ASP}{marek99stable}
\setcitation{satid}{sat/MarienWDB08}
\setcitation{lazyclausegeneration}{constraints/OhrimenkoSC09}
\setcitation{FP}{ACMCS/Hudak89}
\setcitation{GroundingWithBounds}{jair/WittocxMD10}
\setcitation{GroundWithBounds}{jair/WittocxMD10}
\setcitation{SAT}{faia/SilvaLM09}
\setcitation{HandbookOfSAT}{faia/2009-185}
\setcitation{LTC}{iclp/Bogaerts14}
\setcitation{SPSAT}{ictai/DevriendtBMDD12}
\setcitation{BreakID}{sat/DevriendtBBD16}
\setcitation{breakid}{sat/DevriendtBBD16}
\setcitation{LCG}{stuckeyLCG}
\setcitation{MiniZinc}{conf/cp/NethercoteSBBDT07}
\setcitation{minizinc}{conf/cp/NethercoteSBBDT07}
\setcitation{amadini}{cpaior/AmadiniGM13}
\setcitation{bootstrapping}{ngc/BogaertsJDJBD16}
\setcitation{Bootstrapping}{ngc/BogaertsJDJBD16}
\setcitation{GroundedFixpoints}{ai/BogaertsVD15}
\setcitation{PartialGroundedFixpoints}{ijcai/BogaertsVD15}
\setcitation{LogicBlox}{datalog/GreenAK12}
\setcitation{proB}{journals/sttt/LeuschelB08}
\setcitation{NaturalInductions}{KR/DeneckerV14} 
\setcitation{LP}{jacm/EmdenK76}
\setcitation{SMT}{faia/BarrettSST09}
\setcitation{AF}{ai/Dung95}
\setcitation{ADF}{kr/BrewkaW10}
\setcitation{af}{ai/Dung95}
\setcitation{adf}{kr/BrewkaW10}
\setcitation{ADFRevisited}{ijcai/BrewkaSEWW13}
\setcitation{adfrevisited}{ijcai/BrewkaSEWW13}
\setcitation{DefaultLogic}{ai/Reiter80}
\setcitation{DL}{ai/Reiter80}
\setcitation{AEL}{mo85}
\setcitation{minisat}{sat/EenS03}
\setcitation{completion}{adbt/Clark78}
\setcitation{ClarkCompletion}{adbt/Clark78}
\setcitation{wasp}{lpnmr/AlvianoDFLR13}
\setcitation{minisatid}{ictai/DeCat13}
\setcitation{lcg}{stuckeyLCG}
\setcitation{CEGAR}{jacm/ClarkeGJLV03}
\setcitation{cegar}{jacm/ClarkeGJLV03}
\setcitation{CuttingPlane}{or/DantzigFJ54}
\setcitation{kodkod}{tacas/TorlakJ07}
\setcitation{cdcl}{Marques-SilvaS99}
\setcitation{CDCL}{Marques-SilvaS99}
\setcitation{1UIP}{iccad/ZhangMMM01}
\setcitation{relevance}{ijcai/JansenBDJD16}
\setcitation{relevance-implementation}{aspocp/JansenBDJD16}
\setcitation{WFS}{GelderRS91}
\setcitation{wfs}{GelderRS91}
\setcitation{UnfoundedSet}{GelderRS91}
\setcitation{UFS}{GelderRS91}
\setcitation{stablesemantics}{iclp/GelfondL88}
\setcitation{StableSemantics}{iclp/GelfondL88}
\setcitation{shatter}{Shatter}
\setcitation{sbass}{drtiwa11a}
\setcitation{lparsemanual}{url:lparse_manual}
\setcitation{AIC}{ppdp/FlescaGZ04}
\setcitation{templates}{tplp/DassevilleHJD15}
\setcitation{templates2}{iclp/DassevilleHBJD16}
\setcitation{sat-to-sat}{aaai/JanhunenTT16}
\setcitation{sat-to-sat-qbf}{bnp/BogaertsJT16}
\setcitation{sat-to-sat-QBF}{bnp/BogaertsJT16}
\setcitation{sat-to-sat-SO}{kr/BogaertsJT16}
\setcitation{XSB}{SwiW12}
\setcitation{KCmap}{jair/DarwicheM02}
\setcitation{TLA}{DBLP:books/aw/Lamport2002}
\setcitation{EventB}{BookAbrial2010}
\setcitation{MX}{MitchellT05}
\setcitation{MIP}{Sierksma96}
\setcitation{perefectmodel}{minker88/Przymusinski88}
\setcitation{SafeInductions}{ijcai/BogaertsVD17}
\setcitation{AIC}{ppdp/FlescaGZ04}
\setcitation{aic}{ppdp/FlescaGZ04}
\setcitation{alpha}{lpnmr/Weinzierl17}
\setcitation{omiga}{jelia/Dao-TranEFWW12}
\setcitation{gasp}{fuin/PaluDPR09}
\setcitation{asperix}{lpnmr/LefevreN09a}
\setcitation{CTL}{lop/ClarkeE81}
\setcitation{AFT-AIC}{ai/BogaertsC18}
\setcitation{UltimateApproximator}{DeneckerMT04}
\setcitation{KripkeKleene}{Fitting85}
\setcitation{AFT-HO}{corr/CharalambidisRS18} 
\setcitation{HereThere}{Heyting30}
\setcitation{dAEL}{ijcai/HertumCBD16}
\setcitation{SDD}{ijcai/Darwiche11}
\setcitation{HEX}{ijcai/EiterIST05}
\setcitation{wADF}{aaai/BrewkaSWW18}
\setcitation{wADFfix}{corr/BrewkaSWW18}
\setcitation{TransitionSystems}{jacm/NieuwenhuisOT06}
\setcitation{galliwasp}{lopstr/MarpleG12}
\setcitation{GalliWasp}{lopstr/MarpleG12}
\setcitation{clingcon}{tplp/BanbaraKOS17}
\setcitation{lp2sat}{birthday/JanhunenN11}
\setcitation{lp2mip}{LIU12}
\setcitation{lp2diff}{lpnmr/JanhunenNS09}
\setcitation{lp2acyc}{ecai/GebserJR14}
\setcitation{PB}{faia/RousselM09}
\setcitation{pb}{faia/RousselM09}
\setcitation{CuttingPlanes}{dam/CookCT87}
\setcitation{RoundingSAT}{ijcai/ElffersN18}
\setcitation{PRS}{aaai/DixonG02}
\setcitation{sat4j}{jsat/BerreP10}
\setcitation{SAT4J}{jsat/BerreP10}
\setcitation{mingo}{LIU12}
\setcitation{pbmodels}{lpnmr/LiuT05}
\setcitation{HEF-LP}{lpnmr/GebserLL07}
\setcitation{HCF-LP}{amai/Ben-EliyahuD94}
\setcitation{aspcore2}{AspCore2}
\setcitation{}{}
\setcitation{}{}
\setcitation{}{}
\setcitation{}{}
\setcitation{}{}
\setcitation{}{}
\setcitation{}{}
\setcitation{}{}
\setcitation{}{}
\setcitation{}{}
\setcitation{}{}
\setcitation{}{}
\setcitation{}{}
\setcitation{}{}
\setcitation{}{}
\setcitation{}{}
\setcitation{}{}
\setcitation{}{}
\setcitation{}{}
\setcitation{}{}
\setcitation{}{}
\setcitation{}{}
\setcitation{}{}
  
\newcommand\mycite[1]{%
      \ifcsname mycommoncitation#1\endcsname%
   \cite{\getcitation{#1}}%
  \else%
    \cite{#1}%
  \fi%
}	
  
\begin{document}
 
\maketitle

\begin{abstract}
We build on a recently proposed method for explaining solutions of constraint satisfaction problems.
An explanation here is a \textit{sequence} of simple inference steps, where the simplicity of an inference step is measured by the number and types of constraints and facts used, and where the sequence explains all logical consequences of the problem. 
We build on these formal foundations and tackle two emerging questions, namely how to generate explanations that are provably optimal (with respect to the given cost metric) and how to generate them efficiently. 
To answer these questions, we develop 1) an implicit hitting set algorithm for finding \textit{optimal} unsatisfiable subsets; 2) a method to reduce multiple calls for (optimal) unsatisfiable subsets to a single call that takes \emph{constraints} on the subset into account, and 3) a method for re-using relevant information over multiple calls to these algorithms. 
The method is also applicable to other problems that require finding cost-optimal unsatisfiable subsets.
We specifically show that this approach can be used to effectively find sequences of \textit{optimal} explanation steps for constraint satisfaction problems like logic grid puzzles.
\end{abstract}

\section{Introduction}

Building on old ideas to explain domain-specific propagations performed by constraint solvers  \cite{sqalli1996inference,freuder2001explanation}, we recently introduced a 
method that takes as input a satisfiable constraint program and explains the solution-finding process in a human-understandable way  \cite{ecai/BogaertsGCG20}. 
Explanations in that work are sequences of simple inference steps, involving as few constraints and facts as possible. 
The explanation-generation algorithms presented in that work rely heavily on calls for  \emph{Minimal Unsatisfiable Subsets} (MUS) \cite{marques2010minimal} of a derived program, exploiting a one-to-one correspondence between so-called \emph{non-redundant explanations} and MUSs.
The explanation steps in the seminal work are heuristically optimized with respect to a given cost function that should approximate human-understandability, e.g., taking the number of constraints and facts into account, as well as a valuation of their complexity (or priority). 
The algorithm developed in that work has two main weaknesses: first, it provides no guarantees on the quality of the produced explanations due to internally relying on the computation of $\subseteq$-minimal unsatisfiable subsets, which are often suboptimal with respect to the given cost function. 
Secondly, it suffers from performance problems: the lack of optimality is partly overcome by calling a MUS algorithm on increasingly larger subsets of constraints for each candidate implied fact.
However, using multiple MUS calls per literal in each iterations quickly causes efficiency problems, causing the explanation generation process to take several hours.

Motivated by these observations, we develop algorithms that aid explaining CSPs and improve the state-of-the-art in the following ways: 
\begin{itemize}
 \item We develop algorithms that compute (cost-)\textbf{Optimal} Unsatisfiable Subsets (from now on called OUSs) based on the well-known hitting-set duality that is also used for computing cardinality-minimal MUSs \cite{ignatiev2015smallest,DBLP:conf/kr/SaikkoWJ16}.
\item We observe that many of the individual calls for MUSs (or OUSs) can actually be replaced by a single call that searches for an optimal unsatisfiable subset \textbf{among subsets satisfying certain structural constraints}. In other words, we introduce the \emph{Optimal \textbf{Constrained} Unsatisfiable Subsets (OCUS)} problem and we show how $O(n^2)$ calls to MUS/OUS can be replaced by $O(n)$ calls to an OCUS oracle, where $n$ denotes the number of facts to explain. 
\item Finally, we develop techniques for \textbf{optimizing} the O(C)US algorithms further, exploiting domain-specific information coming from the fact that we are in the  \emph{explanation-generation context}. One such optimization is the development of methods for \textbf{information re-use} between consecutive OCUS calls.
\end{itemize}

In this paper, we apply our OCUS algorithms to generate \emph{step-wise} explanations of satisfaction problems. However, MUSs have been used in a variety of contexts, and in particular lie at the foundations of several explanation techniques \cite{junker2001quickxplain,ignatiev2019abduction,schotten}. We conjecture that OCUS can also prove useful in those settings, to take more fine-grained control over which MUSs, and eventually, which explanations are produced.

The rest of this paper is structured as follows.
We discuss background on the hitting-set duality in \cref{sec:background}. \cref{sec:motviation} motivates our work, while \cref{sec:ocus} introduces the OCUS problem and a generic \hitsetbased algorithm for computing OCUSs. In \cref{sec:ocusEx} we show how to optimize this computation in the context of explanations and in  
\cref{sec:experiments}  we experimentally validate the approach.
We discuss related work in  \cref{sec:related} and conclude in \cref{sec:conclusion}.

\section{Background}\label{sec:backgr}\label{sec:background}
We
present all methods using propositional logic but our results easily generalize to richer languages, such as constraint languages, as long as the semantics is given in terms of a satisfaction relation  between expressions in the language and possible states of affairs (assignments of values to variables).
 
%
%
%

Let \voc be a set of propositional symbols, also called \emph{atoms}; this set is implicit in the rest of the paper. A \emph{literal} is an atom $p$ or its negation $\lnot p$. 
A clause is a disjunction of literals. A formula $\formula$ is a conjunction of clauses. 
Slightly abusing notation, a clause is also viewed as a set of literals and a formula as a set of clauses. We use the term clause and constraint interchangeably.
A (partial) interpretation is a consistent (not containing both $p$ and $\lnot p$) set of literals. 
Satisfaction of a formula \formula by an interpretation is defined as usual~\cite{faia/2009-185}.
A \emph{model} of \formula is an interpretation that satisfies \formula; 
$\formula$ is said to be \emph{unsatisfiable} if it has no models.
A literal $l$ is a \emph{consequence} of a formula \formula if $l$ holds in all $\formula$'s models. 
If $I$ is a set of literals, we write \negset{I} for the set of literals $\{\lnot l\mid l\in I\}$.

\begin{definition}

  A \emph{Minimal Unsatisfiable Subset} (MUS) of 
  \F is an unsatisfiable subset $\m{S}$ of $\F$ for which every strict subset of $\m{S} $ is satisfiable. 
%
  \muses{\F} denotes the set of MUSs of \F. 
\end{definition}


\begin{definition}
    A set $\m{S} \subseteq \formula$ is a \emph{Maximal Satisfiable Subset} (MSS) of $ \formula$ if $\m{S}$ is satisfiable and for all $\m{S}'$ with $\m{S}  \subsetneq  \m{S}'\subseteq\formula $, $\m{S}'$ is unsatisfiable.
\end{definition}

\begin{definition}
    A set $\m{S} \subseteq \formula$ is a \emph{correction subset} of \formula if $\formula\setminus\m{S}$ is satisfiable. 
    Such a \m{S} is a \emph{minimal correction subset} (MCS)  of \formula if no strict subset of \m{S} is also a correction subset. 
    \mcses{\F} denotes the set of MCSs of \F. 
\end{definition}
%
Each  MCS of \formula is the complement of an MSS of \formula and vice versa. 

\begin{definition}\label{def:minimal-hs}
    Given a collection of sets $\m{H}$, a hitting set of $\m{H}$ is a set $h$ such that  $h \cap C \neq \emptyset$ for every $C \in \m{H}$. A hitting set is \emph{minimal} if no strict subset of it is also a hitting set.
\end{definition}



The next proposition is the well-known hitting set duality \cite{DBLP:journals/jar/LiffitonS08,ai/Reiter87}  between MCSs and MUSs that forms the basis of our algorithms, as well as algorithms to compute MSSs \cite{DBLP:conf/sat/DaviesB13} and \emph{cardinality-minimal} MUSs \cite{ignatiev2015smallest}.

\begin{proposition}\label{prop:MCS-MUS-hittingset}
%
    A set  $\m{S} \subseteq \formula$ is an MCS of $ \formula$ iff  it is a \emph{minimal hitting set} of \muses{\formula}.
%
    A set  $\m{S} \subseteq \formula$ is a MUS of $ \formula$ iff  it is a \emph{minimal hitting set} of \mcses{\formula}.
\end{proposition}

\section{Motivation}\label{sec:motivation}\label{sec:motviation}
Our work is motivated by the problem of explaining satisfaction problems through a sequence of simple explanation steps. This can be used to teach people problem-solving skills, to compare the difficulty of related satisfaction problems (through the number and complexity of steps needed), and in human-computer solving assistants.

Our original explanation generation algorithm \cite{ecai/BogaertsGCG20} starts from a formula $\formulac$ (in the application coming from a high level CSP), a partial interpretation $I$ (here also viewed as a conjunction of literals) and a cost function $f$ quantifying the difficulty of an explanation step, by means of a weight for every clause and literal in \formula. 

\newcommand\onestep{\ensuremath{\call{explain-One-Step}}\xspace}

\begin{algorithm}[t]
  \DontPrintSemicolon
  
  \caption{$\onestep(\formulac,f,I,\Iend)$}
  \label{alg:oneStep}
$X_{best} \gets \mathit{nil}$\;
\For{$l \in \{\Iend \setminus I\}$}{
    $X \gets \call{MUS}{(\formulac \land I \land \neg l)}$\;
    \If{$f(X)<f(X_{best})$\label{alg:oneStep:ifcheck}}{
        $X_{best} \gets X$\;
    }
}
\Return{$X_{best}$} 
\end{algorithm}

The goal is to find a sequence of \textit{simple} explanation steps, where the simplicity of a step is measured by the total cost of the elements used in the explanation.
An explanation step is an implication $I' \wedge \formulac' \implies N$ where $I'$ is a subset of already derived literals, $\formulac'$ is a subset of constraints of the input formula $\formulac$, and $N$ is a set of literals entailed by $I'$ and $\formulac'$ which are not yet explained.

%
The key part of the algorithm is the search for the next best explanation, given an interpretation $I$ derived so far. 
\cref{alg:oneStep} shows the gist of how this was done.
It takes as input the formula \formulac, a cost function $f$ quantifying the quality of explanations, an interpretation $I$ containing all already derived literals in the sequence so far, and the interpretation-to-explain $\Iend$. 
To compute an explanation, this procedure iterates over the literals that are still to explain, computes for each of them an associated MUS and subsequently selects the lowest cost one from found MUSs.
The reason this works is because there is a one-to-one correspondence between MUSs of $\formulac \land I \land \neg l$ and so-called \emph{non-redundant explanation} of $l$ in terms of (subsets of) $\formulac$ and $I$~\cite{ecai/BogaertsGCG20}. 

Experiments have shown that such a MUS-based approach can easily take hours, especially when multiple MUS calls are performed to increase the chance of finding a good MUS, and hence that algorithmic improvements are needed to make it more practical. 
We see three main points of improvement, all of which will be tackled by our generic OCUS algorithm presented in the next section. 
\begin{itemize}
 \item First of all, since the algorithm is based on \call{MUS} calls, there is no guarantee that the explanation found is indeed optimal 
 (with respect to the given cost function). 
 Performing multiple MUS calls is only a heuristic that is used to circumvent the restriction that \textit{there are no algorithms for cost-based unsatisfiable subset \textbf{optimization}}. 
 \item Second, this algorithm uses \call{MUS} calls for every literal to explain separately. The goal of all these calls is to find a single unsatisfiable subset of $\formulac \land I \land \overline{(\Iend\setminus I)}$ that contains exactly one literal from $\overline{(\Iend\setminus I)}$. This begs the questions whether it is possible \textit{to compute a single (optimal) unsatisfiable subset \textbf{subject to constraints}}, where in our case, the constraint is to include exactly one literal from $\overline{(\Iend\setminus I)}$. 
 \item Finally, the algorithm that computes an entire explanation sequence makes use of repeated calls to \onestep and hence will solve many similar problems. This raises the issue of \textit{\textbf{incrementality}: can we re-use the computed data structures to achieve speed-ups in later calls?}
\end{itemize}

\section{Optimal Constrained Unsatisfiable Subsets} \label{sec:ocus}
The first two considerations from the previous section lead to the following definition. 

\begin{definition}
   Let $\formula$ be a formula, $f:2^{\formula} \to \nat$ a cost function and  $p$ a predicate $p: 2^{\formula}\to \{true,false\}$.  We call 
    $\m{S} \subseteq \formula$ an OCUS of \formula (with respect to $f$ and $p$) if \begin{itemize}                                      
      \item $\m{S}$ is unsatisfiable,
      \item $p(\m{S})$ is true
      \item all other unsatisfiable $\m{S}'\subseteq \formula$ for which $p(\m{S}')$ is true satisfy $f(\m{S}')\geq f(\m{S})$.
    \end{itemize}
\end{definition}

If we assume that the predicate $p$ is specified itself as a CNF over (meta)-variables indicating inclusion of clauses of $\m{F}$, and $f$ is obtained by assigning a weight to each such meta-variable, then the complexity of the problem of finding an OCUS is the same as that of the SMUS (cardinality-minimal MUS) problem  \cite{ignatiev2015smallest}: the associated decision problem is $\Sigma^P_2$-complete. Hardness follows from the fact that SMUS is a special case of OCUS, containment follows - intuitively - from the fact that this can be encoded as an $\exists\forall$-QBF using a Boolean circuit encoding of the costs. 

When considering the procedure \onestep from the perspective of OCUS defined above, the task of the procedure  is to compute an OCUS of the formula $\formula := \formulac\land I\land \overline{\Iend\setminus I}$ with $p$ the predicate that holds for subsets  that contain exactly one literal of $\overline{\Iend\setminus I}$, see \cref{alg:oneStepOCUS}. 

In order to compute an OCUS of a given formula, we propose to build on the hitting set duality of \cref{prop:MCS-MUS-hittingset}. 
For this, we will assume to have access to a solver \cohs that can compute hitting sets of a given collection of sets that are \emph{optimal} (w.r.t.\ a given cost function $f$) among all hitting sets \emph{satisfying a condition $p$}. 
The choice of the underlying hitting set solver will thus determine which types of cost functions and constraints are possible. 
In our implementation, we use a cost function $f$ as well as a condition $p$ that can easily be encoded as linear constraints, thus allowing the use of highly optimized mixed integer programming (MIP) solvers. The \cohs formulation is as follows:
\begin{align*}
\small
  minimize_S \quad & f(S) \\ 
  s.t. \quad & p(S) \\
       & sum(H) \geq 1, \quad &&\forall H \in \setstohit \\
       & s \in \{0,1\}, \quad &&\forall s \in S
\end{align*}
where $S$ is a set of MIP decision variables, one for every clause in $\formula$. In our case, $p$ is expressed as $\sum_{s \in \overline{\Iend\setminus I}} s = 1$. 
%
$f$ is a weighted sum over the variables in $S$. For example, (unit) clauses representing previously derived facts can be given small weights and regular constraints can be given large weights, such that explanations are penalized for including constraints when previously derived facts can be used instead. 
\newcommand\onestepo{\ensuremath{\call{explain-One-Step-ocus}}\xspace}
\begin{algorithm}[t]
  \DontPrintSemicolon
  
  \caption{$\onestepo(\formulac,f,I,\Iend)$}
  \label{alg:oneStepOCUS}
  $p \leftarrow$ exactly one of $\overline{\Iend\setminus I}$\;
  \Return{$\comus(\formulac\land I\land \overline{\Iend\setminus I}, f, p)$} 
\end{algorithm}
\begin{algorithm}[t]
  \DontPrintSemicolon
  $\setstohit  \gets \emptyset$ \; 
  \While{true}{
    $\m{S} \gets \cohs(\setstohit,f,p) $  \;
    \If{ $\lnot \sat(\m{S})$}{\label{alg:ocus-sat-check}
      \Return{$\m{S}$} \;
    }
    $\m{S} \gets  \grow(\m{S},\F) $ \label{line:grow}\;
    $\setstohit  \gets \setstohit  \cup \{  \formula \setminus \m{S}\}$ \;
  }
  \caption{$\comus(\formula,f,p)$ }
  \label{alg:comus}
\end{algorithm}
Our generic algorithm for computing OCUSs is depicted in \cref{alg:comus}. It combines the hitting set-based approach for MUSs of \cite{ignatiev2015smallest} with the use of a MIP solver for (weighted) hitting sets as proposed for maximum satisfiability \cite{DBLP:conf/sat/DaviesB13}. The key novelty is the ability to add structural constraints to the hitting set solver, without impacting the duality principles of \cref{prop:MCS-MUS-hittingset}, as we will show.

Ignoring \cref{line:grow} for a moment, 
the algorithm alternates calls to a hitting set solver with calls to a \sat oracle on a subset $\m{S}$ of $\formula$. 
In case the \sat oracle returns true, i.e., the subset $\m{S}$ is satisfiable, the complement of $\m{S}$ is a correction subset of $\m{F}$ and is added to \setstohit. 

As in the SMUS algorithm of \citet{ignatiev2015smallest}, our algorithm contains an (optional) call to \grow. 
The purpose of the \grow is to expand a satisfiable subset of $\m{F}$ further, to find a smaller correction subset and as such find stronger constraints on the hitting sets. 
In our case, the calls for hitting sets will also take into account the cost ($f$), as well as the meta-level constraints ($p$); as such, it is not clear a priori which properties a good \grow function should have here.
We discuss the different possible implementations of \grow later and evaluate their performance in \cref{sec:experiments}. For correctness of the algorithm, all we need to know is that it returns a satisfiable subset $\m{S}'$ of $\m{F}$ with $\m{S}\subseteq\m{S}'$.
%

Soundness and completeness of the proposal follow from the fact that all sets added to \setstohit are correction subsets, and \cref{thm:soundcomplete}, which states that what is returned is indeed a solution and that a solution will be found if it exists.

\begin{theorem}\label{thm:soundcomplete}
  Let $\m{H}$ be a set of correction subsets of \formula. 
  If $\m{S}$ is a hitting set of \m{H} that is $f$-optimal among the hitting sets of \m{H} satisfying a predicate $p$, and  $\m{S}$ is unsatisfiable, then $\m{S}$ is an OCUS of \formula. 
  
  If  $\m{H}$ has no hitting sets satisfying $p$, then $\formula$ has no OCUSs.
\end{theorem}
\begin{proof}
For the first claim, it is clear that $\m{S}$ is unsatisfiable and satisfies $p$. Hence all we need to show is $f$-optimality of $\m{S}$.
  If there would exist some other unsatisfiable subset $\m{S}'$ that satisfies $p$ with $f(\m{S}')\leq f(\m{S})$, we know that $\m{S}'$ would hit every minimal correction set of \m{F}, and hence also every set in \m{H} (since every correction set is the superset of a minimal correction set).
  Since $\m{S}$ is $f$-optimal among hitting sets of $\m{H}$ satisfying $p$ and $\m{S}'$ also hits $\m{H}$ and satisfies $p$, it must thus be that $f(\m{S})=f(\m{S}')$. 

The second claim immediately follows from \cref{prop:MCS-MUS-hittingset} and the fact that an OCUS is an unsatisfiable subset of $\formula$. 
\end{proof}
%

Perhaps surprisingly, correctness of the proposed algorithm does \emph{not} depend on monotonicity properties of $f$ nor $p$. In principle, any (computable) cost function and condition on the unsatisfiable subsets can be used. In practice however, one is bound by limitations of the chosen hitting set solver.

As an illustration, we now provide an example of one call to $\onestepo$ (Algorithm 
\ref{alg:oneStepOCUS}) and the corresponding \comus-call (Algorithm \ref{alg:comus}) in detail: 
\begin{example}
	Let $\formulac$ be a CNF formula over variables $x_1, x_2, x_3$ with the following four clauses:
		\[ c_1 := \lnot x_1 \vee \lnot x_2 \vee x_3 \qquad  c_2 := \lnot x_1 \vee  x_2 \vee x_3\] \[  c_3 := x_1 \qquad c_4 := \lnot x_2 \vee \lnot x_3 \]
	 The final interpretation $\Iend$ is $\{x_1, \lnot x_2,  x_3\}$. Let the current interpretation $I$ be $\{ x_1\}$, then $\overline{\Iend\setminus I} =  \{x_2, \lnot x_3\}$.
	 
	 To define the input for the OCUS call, we add new clauses representing the already known facts $I$ and the to-be-derived facts $\overline{\Iend\setminus I}$: 
	 \[ c_5 := \{x_1\}\qquad  c_6:=\{x_2\} \qquad  c_7 := \{\lnot x_3\}\]
	 The formula \formula in the \comus-call is thus: 
	 \[\formula= \formulac\land I \land \overline{(\Iend\setminus I)} = \{ c_1 \wedge c_2 \wedge c_3\wedge c_4\wedge c_5\wedge c_6\wedge c_7\}\]	
	 	 
	 We define $p\triangleq$ \textit{exactly-one$(c_6, c_7)$} and $f = \sum w_ic_i$ with clause weights $w_1 = 60, w_2=60, w_3=100, w_4=100, w_5=1, w_6=1, w_7=1$.

\setstohit is initialized as the empty set. At each iteration, the hitting set solver will search for a cost-minimal assignment that hits all sets in \setstohit and that furthermore contains exactly one of $c_6$ and $c_7$ (due to $p$).
	Table \ref{tab:explanation-steps-expanded} shows the computed steps in the different iterations of Algorithm~\ref{alg:comus} given the above input.
\begin{table*}[!t]
	\centering
			\begin{tabular}{lcccc} 
				&$\m{S}$ & $\sat(\m{S})$ & \grow($\m{S}$, $\formula$) & $\setstohit  \gets \setstohit  \cup \{  \formula \setminus \m{S}\}$\\ 
				\toprule[2pt]
				1 &$ \emptyset $ & $\mathit{true}$ & $\{c_1, c_2, c_3, c_4, c_5\}$ & $\{ \{c_6, c_7\}\}$   \\
				\midrule	
				2 &$\{c_6\}$ & $\mathit{true}$ & $\{c_1, c_2, c_3, c_5, c_6\}$ & $\{ \{c_6, c_7\}, \{c_4, c_7\}\}$  \\  
				\midrule
				3 &$\{c_7\}$ & $\mathit{true}$ & $\{c_1, c_3, c_4, c_5, c_7\}$ & $\{\{c_6, c_7\}, \{c_4, c_7\}, \{c_2, c_6\}\}$  \\
				\midrule
				\multirow{2}{*}{4}& \multirow{2}{*}{$\{c_2, c_7\}$}  & \multirow{2}{*}{$\mathit{true}$} & \multirow{2}{*}{$\{c_2, c_3, c_4, c_5, c_6, c_7 \}$ } & $\{\{c_6, c_7\}, \{c_4, c_7\}, \{c_2, c_6\}, $  \\ 
				&   & &   & $ \{c_1\}\}$  \\
				\midrule
				\multirow{2}{*}{5}& \multirow{2}{*}{$\{c_1, c_2, c_7\}$}  & \multirow{2}{*}{$\mathit{true}$} & \multirow{2}{*}{$\{c_1, c_2, c_4, c_6, c_7 \}$ } & $\{\{c_6, c_7\}, \{c_4, c_7\}, \{c_2, c_6\}, $  \\ 
				&   & &   & $ \{c_1\}, \{c_3, c_5\}\}$  \\
				\midrule
				6&  $\{ c_1, c_2, c_5, c_7 \}$ & $\mathit{false}$ & & \\
			\end{tabular}
	\caption{Example of an OCUS-explanation computation.}
	\label{tab:explanation-steps-expanded}
\end{table*}
\emph{In this example, the \grow we used is the one called \emph{Max-Actual-Unif} in \cref{sec:experiments}. }
\end{example}

\section{Efficient OCUS Computation for Explanations}\label{sec:ocusEx}
Algorithm~\ref{alg:comus} is generic and can also be used to find (unconstrained) \omus{}s, namely with a trivially true $p$.
However, its constrainedness property allows to remove the need to compute a MUS/\omus for every literal. This decreases the complexity of explanation sequence generation from $O(n^2)$ calls to MUS to $O(n)$ calls to OCUS, namely, once for every step in the sequence. 

We now discuss optimizations to the OCUS algorithm that are specific to explanation sequence generation, though they can also be used when other forms of domain knowledge are present. 
 
\paragraph{Incremental OCUS Computation.}
Inherently, generating a sequence of explanations still requires as many OCUS calls as there are literals to explain. 
Indeed, a greedy sequence construction algorithm 
calls \onestepo iteratively with a growing interpretation $I$ until $I=\Iend$.

All of these calls to \onestepo, and hence OCUS, are done with very similar input (the set of constraints does not change, and the $I$ slowly grows between two calls). For this reason, it makes sense that information computed during one of the earlier stages can be useful in later stages as well. 

The main question is, suppose two \comus calls are done, first with inputs $\formula_1$, $f_1$, and $p_1$, and later with $\formula_2$, $f_2$, and $p_2$; how can we make use as much as possible of the data computations of the first call to speed-up the second call? The answer is surprisingly elegant. The most important data \comus keeps track of  is the collection \setstohit of correction subsets that need to be hit.

This collection in itself is not useful for transfer between two calls, since -- unless we assume that $\formula_2$ is a subset of $\formula_1$, there is no reason to assume that a set in $\setstohit_1$ should also be hit in the second call. 
However, each set $H$ in $\setstohit$ is the complement (with respect to the formula at hand) of a \emph{satisfiable subset} of constraints, and this satisfiability remains true. 
Thus, instead of storing $\setstohit$, we can keep track of a set \satsets of \emph{satisfiable subsets} (the sets $\m{S}$ in the \comus algorithm). 
When a second call to \comus is performed, we can then initialize $\setstohit$ as the complement of each of these satisfiable subsets with respect to $\formula_2$, i.e., \[\setstohit\gets \{\formula_2\setminus \m{S}\mid \m{S}\in \satsets\}.\]

The effect of this is that we \textit{bootstrap} the hitting set solver with an initial set $\setstohit$.

For hitting set solvers that natively implement incrementality, we can generalize this idea further: we know that all calls to $\comus(\formula,f,p)$ will be cast with $\formula \subseteq \m{C}\cup \Iend \cup \overline{\Iend \setminus I_0}$, where $I_0$ is the start interpretation. Since our implementation uses a MIP solver for computing hitting sets (see Section~\ref{sec:backgr}), and we have this upper bound on the set of formulas to be used, we can initialize all relevant decision variables once. To compute the conditional hitting set for a specific $\formulac\cup I\cup \overline{\Iend\setminus I} \subseteq \m{C}\cup \Iend \cup \overline{\Iend \setminus I_0}$ we need to ensure that the MIP solver only uses literals in $\formulac\cup I\cup \overline{\Iend\setminus I}$, for example by giving all other literals infinite weight in the cost function. In this way, the MIP solver will automatically maintain and reuse previously found sets-to-hit in each of its computations. 

\paragraph{Explanations with Bounded OUS.}
Instead of working \comus-based, we can now also generate optimal explanations by replacing the MUS call by an \omus call  in Algorithm~\ref{alg:oneStep} (where OUS is computed as in Algorithm~\ref{alg:comus}, but with a trivially true $p$). 
When doing this, we know that every \comus of cost greater than or equal to $f(X_{\mathit{best}})$ will be discarded by the check on Line 4 of Algorithm~\ref{alg:oneStep}.
As such, a next optimization is to, instead of searching for an OUS, perform a \emph{bounded OUS check}, which only computes an OUS in case one of cost smaller than a given bound $\mathit{ub}$ exists.  
In our specific implementation, bounded \omus is performed by interrupting this \omus-call (after Line 3 Algorithm~\ref{alg:comus}) if $f(\m{S}) > \mathit{ub}$.

Since the bounding on the \omus cost has the most effect if cheap \omus{}s are found early in the loop across the different literals, we keep track of an upper bound of the cost of an OUS for each literal to explain. This is initialized to a value greater than any \omus, e.g., as $f(\formulac \land I_0 \land \overline{\Iend \setminus I_0})$, and is updated every time an \omus explaining that literal is found; when going through the loop in Line 2 of Algorithm~\ref{alg:comus}, we then handle literals in order of increasing upper bounds.

\paragraph{Domain-Specific Implementations of \grow.} \label{para:domainspecificgrow}

The goal of the \grow procedure is to turn $\m{S}$ into a larger subformula of $\formula$. The effect of this is that the complement added to \setstohit will be smaller, and hence, a stronger restriction for the hitting set solver is found.  

Choosing an effective \grow procedure requires finding a difficult balance: on the one hand, we want our subformula to be as large as possible (which ultimately would correspond to computing the maximum satisfiable subformula), 
but on the other hand we also want the procedure to be very efficient as it is called in every iteration. 

For the case of explanations we are in, we make the following observations: 
\begin{itemize}
 \item Our formula at hand (using the notation from the \onestepo algorithm) consists of three types of clauses: \begin{inparaenum}\item  (translations of) the problem constraints (this is \formulac) \item literals representing the assignment found thus far (this is $I$), and \item the negations of literals not-yet-derived (this is $\overline{\Iend\setminus I}$). \end{inparaenum}
 \item $\formulac$ and $I$ together are satisfiable, with assignment $I_{end}$, and \emph{mutually supportive}, by this we mean that making more constraints in \formulac true, more literals in $I$ will automatically become true and vice versa. 
 \item The constraint $p$ enforces that each hitting set will contain \textbf{exactly} one literal of  $\overline{\Iend\setminus I}$
\end{itemize}
Since the restriction on the third type of elements of $\formula$ are already strong, it makes sense to use the \grow(\m{S},\F) procedure to search for a \emph{maximal} satisfiable subset of $\formulac\cup I$ with hard constraints that $\m{S}$ should be satisfied, using a call to an efficient  (partial) \maxsat solver. Furthermore, we can initialize this call as well as any call to a \sat solver with the polarities for all variables set to the value they take in $\Iend$. 

We evaluate different grow strategies in the experiments section, including the use of partial \maxsat as well as weighted partial \maxsat based on the weights in the cost function $f$.

\paragraph{Example 1 (cont.)} Consider line 0 in table \ref{tab:explanation-steps-expanded}. During the \grow procedure, the \maxsat solver \emph{Max-Actual-Unif} with polarities set to \Iend branches when multiple assignment to a literal are possible. By hinting the polarities of the literals, we guide the solver and it assigns all values according to the end interpretation and neither $c_6$ nor $c_7$ is taken.


\section{Experiments}\label{sec:experiments}

We now experimentally validate the performance of the different versions of our algorithm.
Our benchmarks were run on a compute cluster, where each explanation sequence generation was assigned a single core on a 10-core INTEL Xeon Gold 61482 (Skylake) processor, a timelimit of 120 minutes and a memory-limit of 4GB. 
Everything was implemented in Python on top of PySAT\footnote{\url{https://pysathq.github.io}} and is available at \url{https://github.com/ML-KULeuven/ocus-explain}. 
For MIP calls, we used Gurobi 9.0, for SAT calls MiniSat 2.2 and for MaxSAT calls RC2 as bundled with PySAT (version 0.1.6.dev11). In the MUS-based approach we used PySAT's deletion-based MUS extractor MUSX~\cite{marques2010minimal}.

All of our experiments were run on a direct translation to PySAT of the 10 puzzles of \citet{ecai/BogaertsGCG20}\footnote{In one of the puzzles, an error in the automatic translation of the natural language constraints was found and fixed.}. 
We used a cost of 60 for puzzle-agnostic constraints; 100 for puzzle-specific constraints; and cost 1 for facts.
When generating an explanation sequence for such puzzles, the unsatisfiable subset identifies which constraints and which previously derived facts should be combined to derive new information. 
Our experiments are designed to answer the following research questions: 
\begin{description}
 \item[Q1] What is the effect of requiring optimality of the generated MUSs on the \textbf{quality} of the generated explanations? 
 \item[Q2] Which \textbf{domain-specific \grow methods} perform best?
 \item[Q3] What is the effect of the use of \textbf{constrainedness} on the time required to compute an explanation sequence?
 \item[Q4] Does \textbf{re-use} of information across the different iterations improve efficiency?
\end{description}


\begin{figure}[t]
	\centering
	\includegraphics[width=0.8\columnwidth]{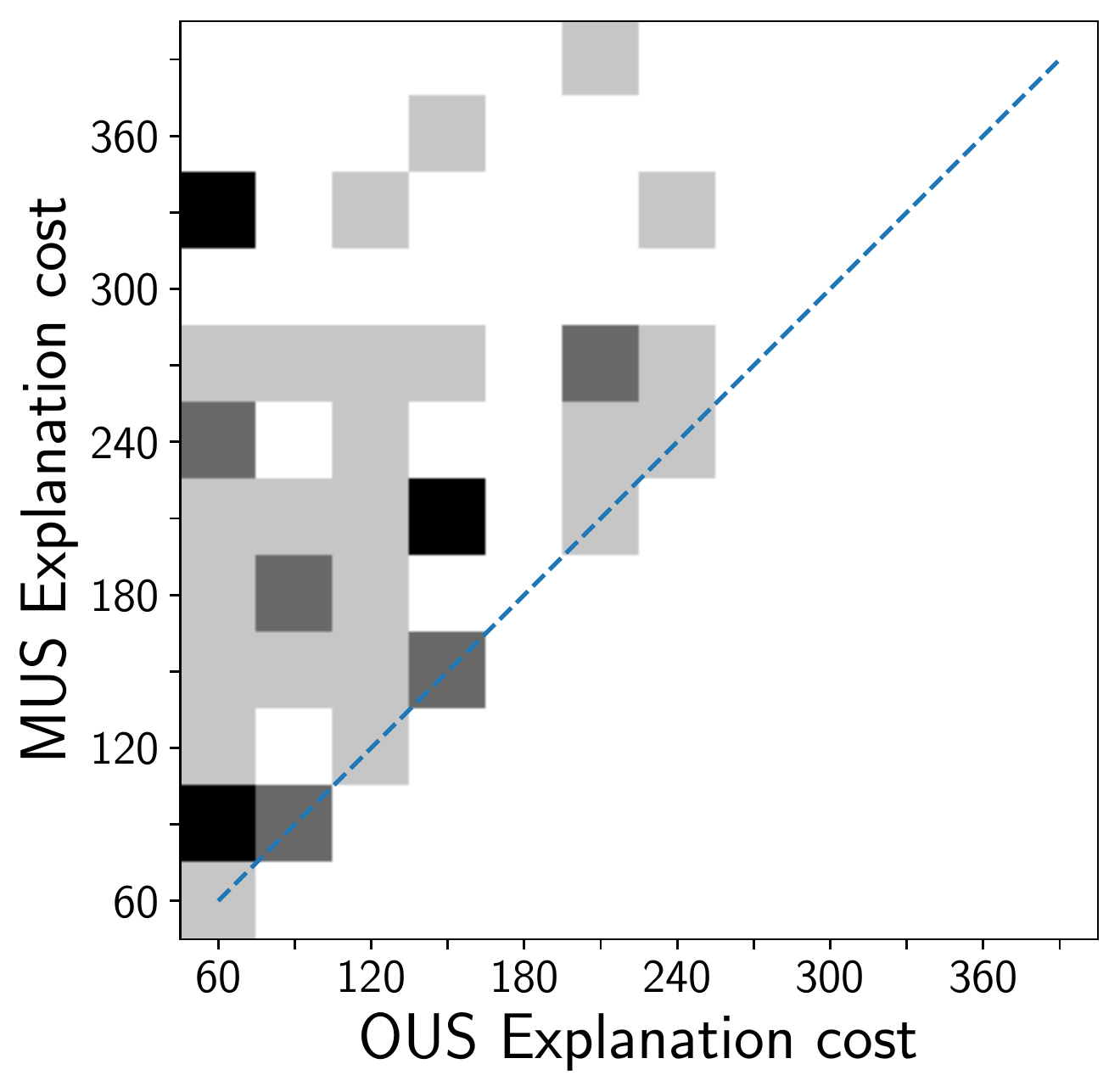}
	\caption{Q1 - Explanation quality comparison of optimal versus subset-minimal explanations in the generated puzzle explanation sequences.}
	\label{fig:rq1_heatmap}
\end{figure}

\paragraph{Explanation quality.}\label{paragraph:explanationquality}
To evaluate the effect of optimality on the quality of the generated explanations, we reimplemented a MUS-based explanation generator based on \cref{alg:oneStep}. 
Before presenting the results, we want to stress that this is \emph{not} a fair comparison with the implementation of \citet{ecai/BogaertsGCG20}, since there -- in order to avoid the quality problems we will illustrate below --  an extra inner loop was used that employs \emph{even more} calls to \call{MUS} for a selected set of subsets of \formulac of increasing size. 
While this yields better explanations, it comes at the expense of computation time, thereby leading to several hours to generate the explanation of a single puzzle. 

To answer \textbf{Q1}, we ran the \call{MUS}-based algorithm as described in \cref{alg:oneStep} and compared at every step the cost of the produced explanation with the cost of the optimal explanation. 
%
These costs are plotted on a heatmap in Figure \ref{fig:rq1_heatmap}, where the darkness represents the number of occurrences of the combination at hand. 
We see that the difference in quality is striking in many cases, with the MUS-based solution often missing very cheap explanations (as seen by the two dark squares in the column around cost 60), thereby confirming the need for a cost-based \omus/\comus approach.

\paragraph{Domain-specific \grow.} 
In our OCUS algorithm, we do not just aim to find any satisfiable subsets, but we prefer \emph{high quality} satisfiable subsets: subsets that impose strong constraints on the assignments the optimal hitting set solver can find. 
This induces a trade-off between \emph{efficiency} of the \grow strategy and \emph{quality} of the produced satisfiable subset.


Thus, to answer \textbf{Q2}, we compared variations of  \comus that only differ in which \grow strategy they use. 
Figure \ref{fig:grow_strategies} depicts the average (over all the puzzles) cumulative explanation time to derive a number of literals. Note that most puzzles only contain 150 literals, except for 2, which contain 96 and 250 literals respectively. When a method times out for a puzzle at one step, a runtime value of 7200 is used in computing the averages for all future steps.
The configurations used are as follows:
\begin{itemize}
\item \emph{Max} refers to growing with a MaxSAT solver and \emph{Greedy} to growing using a heuristic method implemented by repeated sat calls, while 
\emph{no-grow} refers to skipping the \grow step.
\item  \emph{Full} refers to using the full unsatisfiable formula $\mathcal{F}$  while \emph{Actual} refers to using only the constraints that hold in the final interpretation (see Section~\ref{sec:ocusEx}). For instance, for the MaxSAT-based calls, \emph{Actual} means that only the previously derived facts and the original constraints are taken into account when computing optimality. 
\item The \maxsat solver (\emph{Max}) is combined with different weighing schemes: uniform weights (\emph{unif}), cost-function weights (\emph{pos}) (equal to the weights in $f$),  or the inverse of these costs (\emph{inv}) defined as $\max_j(w_j) + 1 - w_i$.
\end{itemize}

We can observe that not using a grow strategy performs badly, as do weighted MaxSAT grows with costs $w_i$ (-Pos). Greedy grow strategies improve on not using a grow strategy, but not substantially. The two approaches that work best use the domain-specific knowledge of doing a MaxSAT grow on $\formulac\cup I$, with the unweighted variant the only one that never times out.
\begin{figure}[t]
	\centering
	\includegraphics[width=\columnwidth]{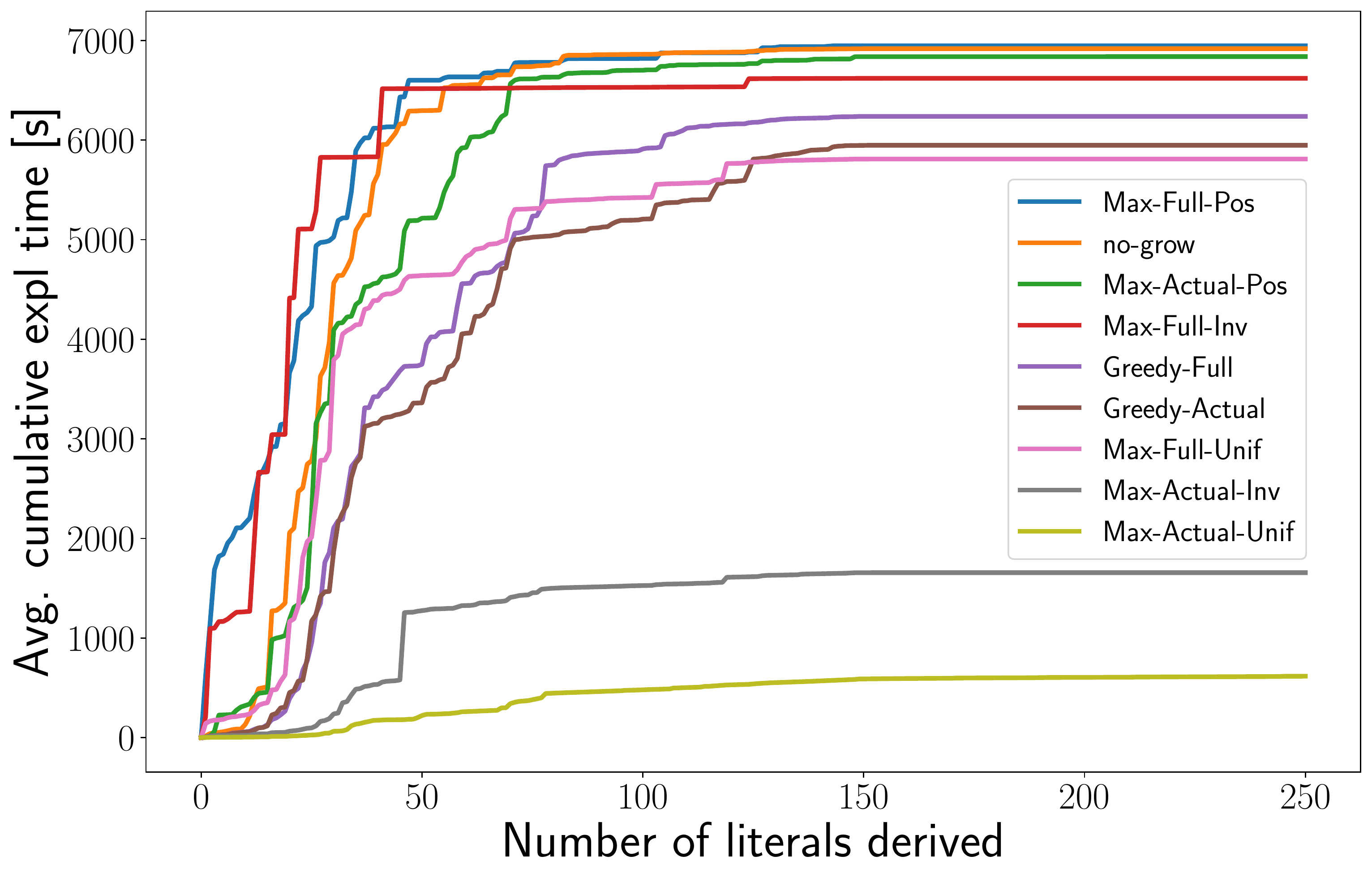}
	\caption{Q2 - Explanation specific \grow strategies for \comus.}
	\label{fig:grow_strategies}
\end{figure}

\paragraph{Constrainedness and incrementality.}
To answer \textbf{Q3} and \textbf{Q4}, we compare the effect of constrainedness in the search for explanations (C) and incrementality. Next to OCUS, we also include the bounded OUS approach (OUSb), where we call the OUS algorithm for every literal in every step, but we reuse information by giving it the current best bound $f(X_{best})$ and iterating over the literals that performed best in the previous call first. Based on the previous experiment, we always use (both for OUSb and OCUS) \emph{Max-Actual-Unif} as grow strategy.
For \emph{OCUS}, incrementality (\emph{+Incr.~HS}) is achieved by reusing the same incremental MIP hitting set solver throughout the explanation calls, as explained in Section~\ref{sec:ocusEx}.  
To have a better view on how incrementality also affects OUSb, we add incrementality to it in the following ways:
\begin{itemize}
	\item \emph{SS.~caching} keeps track of the satisfiable subsets, which are used to initialize \setstohit for a fresh hitting set solver instance each time.
	\item \emph{Lit.~Incr.~HS} uses a separate incremental hitting set solver for every literal to explain, throughout the explanation calls. Once the literal is explained, the hitting set solver is discarded.
\end{itemize}
%
Figure \ref{fig:incrementality_constraindness} shows the results, where 
In this figure, 
the configurations are compared in a similar fashion to Figure \ref{fig:grow_strategies}.

\begin{figure}[t]
	\centering
	\includegraphics[width=\columnwidth]{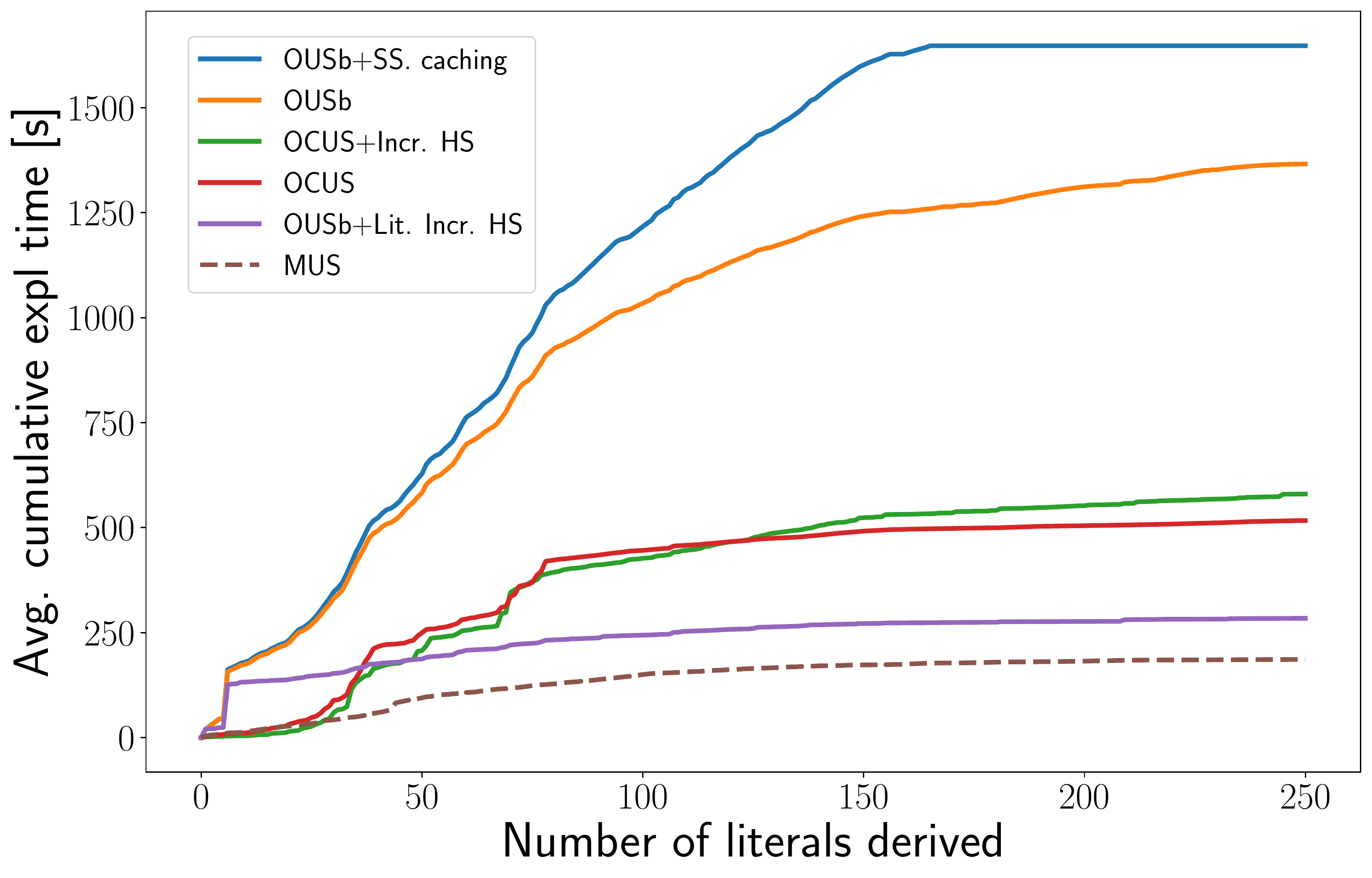}
	\caption{Q3 - Cumulative runtime evolution enhancements on incrementality and constrainedness.}
	\label{fig:incrementality_constraindness}
\end{figure}

When comparing the configurations, we see that the plain \emph{MUS}-based implementation is faster than the \emph{O(C)US} implementations, as it solves a simpler problem (with worse quality results as shown in \textbf{Q1}). 
Replacing MUS by bounded OUS calls (orange line) leads to a much larger computation cost. The generic SS caching technique adds additional overhead. 

The OCUS variants significantly improve runtime compared to those two bounded OUS approaches, by reducing the number of OUS calls. For OCUS, using an incremental hitting set solver across all steps seems to be slightly faster for deriving literals earlier in the sequence, while inducing a small overhead for the literals at the end of the sequence.

When looking at the runtime to explain the entire sequence, best results are obtained with OUSb + Lit. Incr. HS, that is, using an incremental hitting set solver \textit{for every individual literal to explain}. However, for the first literals, we can see that it takes much more computation time and that reducing the number of OUS calls from $n$ to $1$ per explanation step improves runtime (OCUS). However, making each of the $n$ calls incremental and bounded across the entire explanation sequence generation process leads to an even faster process overall (OUSb+Lit. Incr. HS).

\section{Related Work}\label{sec:related}

In the last few years, driven by the increasingly many successes of Artificial Intelligence (AI), there is a growing need for \textbf{eXplainable Artificial Intelligence (XAI)}~\cite{miller2019explanation}.
In the research community, this need manifests itself through the emergence of (interdisciplinary) workshops and conferences on this topic~\cite{xai-ijcai,FAT} and American and European incentives to stimulate research in the area~\cite{gunning2017explainable,hamonrobustness,fetproact}. 

While the main focus of XAI research has been on explaining black-box machine learning systems \cite{lundberg2017unified,guidotti2018survey,ignatiev2019abduction}, also model-based systems, which are typically considered more transparent, are in need of explanation mechanisms. 
Indeed, by advances in solving methods in research fields such as constraint programming \cite{fai/Rossi06} and SAT \cite{faia/2009-185}, as well as by hardware improvement, such systems now easily consider millions of alternatives in short amounts of time. 
Because of this complexity, the question arises how to generate human-interpretable explanations of the conclusions they make. 
Explanations for model-based systems have been considered mostly for explain \textit{unsatisfiable} problem instances~\cite{junker2001quickxplain}, and have recently seen a rejuvenation in various subdomains of constraint reasoning \cite{fox2017explainable,vcyras2019argumentation,chakraborti2017plan,ecai/BogaertsGCG20}.

In this context, we recently introduced \emph{step-wise explanations} \cite{ecai/BogaertsGCG20} and applied them to Zebra puzzles; similar explanations, but for a wider range of puzzles, have been investigated by \citet{schotten}. 
Our current work is motivated by a concrete algorithmic need: to generate these explanations efficiently, we need algorithms that can find optimal MUSs with respect to a given cost function, where the cost function approximates human-understandability of the corresponding explanation step. 
The closest related works can be found in the literature on generating or enumerating MUSs \cite{conf/sat/LynceM04,liffiton2016fast}.
Different techniques are employed to find MUSs, including  manipulating resolution proofs produced by SAT solvers \cite{goldberg,DBLP:journals/fmsd/GershmanKS08,DBLP:conf/sat/DershowitzHN06}, incremental solving to enable/disable clauses and branch-and-bound search \cite{DBLP:conf/dac/OhMASM04}, or by BDD-manipulation methods \cite{huang}.
Other methods work by means of translation into a so-called Quantified \maxsat \cite{DBLP:journals/constraints/IgnatievJM16}, a field that combines the expressivity of Quantified Boolean Formulas (QBF) \mycite{QBF} with optimization as known from \maxsat \mycite{DBLP:series/faia/LiM09}, or by exploiting the so-called hitting set duality \cite{ignatiev2015smallest} bootstrapped using MCS-enumeration \cite{marques2020reasoning}. 
An \textit{abstract} framework for describing \hitsetbased algorithms, including optimization was developed by \citet{DBLP:conf/kr/SaikkoWJ16}. While our approach can be seen to fit the framework, the terminology is focused on MaxSAT rather than MUS and would complicate our exposition.
To the best of our knowledge, only few have considered \emph{optimizing} MUSs: the only criterion considered yet is cardinality-minimality \cite{conf/sat/LynceM04,ignatiev2015smallest}. 

\ignore{
Our paper builds on the algorithm of \citet{ignatiev2015smallest}, which fits in a general class of so-called \emph{implicit hitting set algorithms}.
While these algorithms find their root in early work of \citet{ai/Reiter87}, they only really boosted in popularity when applied in the context of \maxsat solving \cite{DBLP:conf/cp/DaviesB11,DBLP:conf/sat/DaviesB13,davies}, where \hitsetbased solvers are often among the best solvers in the competitions. 
}



\section{Conclusion, Challenges and Future work}\label{sec:conclusion}
We presented a \hitsetbased algorithm for finding \textit{optimal constrained} unsatisfiable subsets, with an application in generating explanation sequence for constraint satisfaction problems.
We extended our methods with \emph{incrementality}, as well as with a \emph{domain-specific method for extending satisfiable subsets (\grow)}. 
This domain-specific \grow method was key to generating explanation sequences in a reasonable amount of time.
We noticed that, independently, incrementality and constrainedness have major benefits on explanation-generation time. The best method on the tested puzzles was the incremental, bounded but non-constrained variant. 
It remains an open question how to make constrainedness and incrementality work together more effectively, as well as how to further optimize O(C)US-based explanations, for instance using disjoint MCS enumeration \cite{marques2020reasoning}.

With the observed impact of different `\grow' methods, an open question remains whether we can quantify precisely and in a generic way what a \textit{good} or even the best set-to-hit is in a hitting set approach. 
The synergies of our approach with the more general problem of QMaxSAT \cite{DBLP:journals/constraints/IgnatievJM16} is another open question.


The concept of bounded (incremental) OUS and OCUS are not limited to explanations of satisfaction problems and we are keen to explore other applications too.
A general direction here are explanations of \textit{optimisation} problems and the role of the objective function in explanations.
\section*{Acknowledgments}
This research received partial funding
from the Flemish Government (AI Research Program); the FWO Flanders project G070521N; and funding from the European Research Council (ERC) under the European Union’s Horizon 2020 research and innovation program (Grant No.~101002802, CHAT-Opt).

{
%
\bibliographystyle{named}
\bibliography{krrlib,ref} 

\begin{thebibliography}{}

\bibitem[\protect\citeauthoryear{Biere \bgroup \em et al.\egroup
  }{2009}]{faia/2009-185}
Armin Biere, Marijn Heule, Hans {van Maaren}, and Toby Walsh.
\newblock {\em Handbook of Satisfiability}.
\newblock 2009.

\bibitem[\protect\citeauthoryear{Bogaerts \bgroup \em et al.\egroup
  }{2020}]{ecai/BogaertsGCG20}
Bart Bogaerts, Emilio Gamba, Jens Claes, and Tias Guns.
\newblock Step-wise explanations of constraint satisfaction problems.
\newblock In {\em Proceedigns of {ECAI}}, pages 640--647, 2020.

\bibitem[\protect\citeauthoryear{Chakraborti \bgroup \em et al.\egroup
  }{2017}]{chakraborti2017plan}
Tathagata Chakraborti, Sarath Sreedharan, Yu~Zhang, and Subbarao Kambhampati.
\newblock Plan explanations as model reconciliation: moving beyond explanation
  as soliloquy.
\newblock In {\em Proceedings of {IJCAI}}, pages 156--163, 2017.

\bibitem[\protect\citeauthoryear{{\v{C}}yras \bgroup \em et al.\egroup
  }{2019}]{vcyras2019argumentation}
Kristijonas {\v{C}}yras, Dimitrios Letsios, Ruth Misener, and Francesca Toni.
\newblock Argumentation for explainable scheduling.
\newblock In {\em Proceedings of {AAAI}}, pages 2752--2759, 2019.

\bibitem[\protect\citeauthoryear{Davies and
  Bacchus}{2013}]{DBLP:conf/sat/DaviesB13}
Jessica Davies and Fahiem Bacchus.
\newblock Exploiting the power of {MIP} solvers in {MAXsat}.
\newblock In {\em Proceedings of {SAT}}, pages 166--181, 2013.

\bibitem[\protect\citeauthoryear{Dershowitz \bgroup \em et al.\egroup
  }{2006}]{DBLP:conf/sat/DershowitzHN06}
Nachum Dershowitz, Ziyad Hanna, and Alexander Nadel.
\newblock A scalable algorithm for minimal unsatisfiable core extraction.
\newblock In {\em Proceedings of {SAT}}, pages 36--41, 2006.

\bibitem[\protect\citeauthoryear{Espasa \bgroup \em et al.\egroup
  }{2021}]{schotten}
Joan Espasa, Ian~P. Gent, Ruth Hoffmann, Christopher Jefferson, and Alice~M.
  Lynch.
\newblock Using small muses to explain how to solve pen and paper puzzles.
\newblock {\em ArXiv}, abs/2104.15040, 2021.

\bibitem[\protect\citeauthoryear{{FET}}{2019}]{fetproact}
Fetproact-eic-05-2019, fet proactive: emerging paradigms and communities, call,
  2019.
\newblock Horizon 2020 Framework Programme.

\bibitem[\protect\citeauthoryear{Fox \bgroup \em et al.\egroup
  }{2017}]{fox2017explainable}
Maria Fox, Derek Long, and Daniele Magazzeni.
\newblock Explainable planning.
\newblock In {\em Proceedings of {IJCAI'17-XAI}}, 2017.

\bibitem[\protect\citeauthoryear{Freuder \bgroup \em et al.\egroup
  }{2001}]{freuder2001explanation}
Eugene~C Freuder, Chavalit Likitvivatanavong, and Richard~J Wallace.
\newblock Explanation and implication for configuration problems.
\newblock In {\em IJCAI 2001 workshop on configuration}, pages 31--37, 2001.

\bibitem[\protect\citeauthoryear{Gershman \bgroup \em et al.\egroup
  }{2008}]{DBLP:journals/fmsd/GershmanKS08}
Roman Gershman, Maya Koifman, and Ofer Strichman.
\newblock An approach for extracting a small unsatisfiable core.
\newblock {\em Formal Methods in System Design}, 33(1-3):1--27, 2008.

\bibitem[\protect\citeauthoryear{Goldberg and Novikov}{2003}]{goldberg}
Evgueni Goldberg and Yakov Novikov.
\newblock Verification of proofs of unsatisfiability for {CNF} formulas.
\newblock In {\em Proceedings of {DATE}}, pages 10886--10891, 2003.

\bibitem[\protect\citeauthoryear{Guidotti \bgroup \em et al.\egroup
  }{2018}]{guidotti2018survey}
Riccardo Guidotti, Anna Monreale, Salvatore Ruggieri, Franco Turini, Fosca
  Giannotti, and Dino Pedreschi.
\newblock A survey of methods for explaining black box models.
\newblock {\em ACM computing surveys (CSUR)}, 51(5):1--42, 2018.

\bibitem[\protect\citeauthoryear{Gunning}{2017}]{gunning2017explainable}
David Gunning.
\newblock Explainable artificial intelligence (xai).
\newblock {\em Defense Advanced Research Projects Agency}, 2, 2017.

\bibitem[\protect\citeauthoryear{Hamon \bgroup \em et al.\egroup
  }{2020}]{hamonrobustness}
Ronan Hamon, Henrik Junklewitz, and Ignacio Sanchez.
\newblock Robustness and explainability of artificial intelligence.
\newblock {\em Publications Office of the European Union}, 2020.

\bibitem[\protect\citeauthoryear{Hildebrandt \bgroup \em et al.\egroup
  }{2020}]{FAT}
Mireille Hildebrandt, Carlos Castillo, Elisa Celis, Salvatore Ruggieri, Linnet
  Taylor, and Gabriela Zanfir{-}Fortuna, editors.
\newblock {\em Proceedings of {FAT*}}, 2020.

\bibitem[\protect\citeauthoryear{Huang}{2005}]{huang}
Jinbo Huang.
\newblock Mup: A minimal unsatisfiability prover.
\newblock In {\em Proceedings of the Asia and South Pacific Design Automation
  Conference, ASP-DAC}, pages 432-- 437 Vol. 1, 2005.

\bibitem[\protect\citeauthoryear{Ignatiev \bgroup \em et al.\egroup
  }{2015}]{ignatiev2015smallest}
Alexey Ignatiev, Alessandro Previti, Mark Liffiton, and Joao Marques-Silva.
\newblock Smallest {MUS} extraction with minimal hitting set dualization.
\newblock In {\em Proceedings of {CP}}, 2015.

\bibitem[\protect\citeauthoryear{Ignatiev \bgroup \em et al.\egroup
  }{2016}]{DBLP:journals/constraints/IgnatievJM16}
Alexey Ignatiev, Mikol{\'{a}}s Janota, and Jo{\~{a}}o Marques{-}Silva.
\newblock Quantified maximum satisfiability.
\newblock {\em Constraints}, 21(2):277--302, 2016.

\bibitem[\protect\citeauthoryear{Ignatiev \bgroup \em et al.\egroup
  }{2019}]{ignatiev2019abduction}
Alexey Ignatiev, Nina Narodytska, and Joao Marques-Silva.
\newblock Abduction-based explanations for machine learning models.
\newblock In {\em Proceedings of {AAAI}}, pages 1511--1519, 2019.

\bibitem[\protect\citeauthoryear{Junker}{2001}]{junker2001quickxplain}
Ulrich Junker.
\newblock {QuickXPlain}: Conflict detection for arbitrary constraint
  propagation algorithms.
\newblock In {\em IJCAI'01 Workshop on Modelling and Solving problems with
  constraints}, 2001.

\bibitem[\protect\citeauthoryear{Kleine{ }B{\"{u}}ning and Bubeck}{2009}]{QBF}
Hans Kleine{ }B{\"{u}}ning and Uwe Bubeck.
\newblock Theory of quantified boolean formulas.
\newblock In {\em Handbook of Satisfiability}, pages 735--760. 2009.

\bibitem[\protect\citeauthoryear{Li and
  Many{\`{a}}}{2009}]{DBLP:series/faia/LiM09}
Chu~Min Li and Felip Many{\`{a}}.
\newblock {MaxSAT}, hard and soft constraints.
\newblock In {\em Handbook of Satisfiability}, pages 613--631. 2009.

\bibitem[\protect\citeauthoryear{Liffiton and
  Sakallah}{2008}]{DBLP:journals/jar/LiffitonS08}
Mark~H. Liffiton and Karem~A. Sakallah.
\newblock Algorithms for computing minimal unsatisfiable subsets of
  constraints.
\newblock {\em J. Autom. Reasoning}, 40(1):1--33, 2008.

\bibitem[\protect\citeauthoryear{Liffiton \bgroup \em et al.\egroup
  }{2016}]{liffiton2016fast}
Mark~H Liffiton, Alessandro Previti, Ammar Malik, and Joao Marques-Silva.
\newblock Fast, flexible mus enumeration.
\newblock {\em Constraints}, 21(2):223--250, 2016.

\bibitem[\protect\citeauthoryear{Lundberg and Lee}{2017}]{lundberg2017unified}
Scott~M Lundberg and Su-In Lee.
\newblock A unified approach to interpreting model predictions.
\newblock In {\em Proceedings of {NIPS}}, pages 4765--4774, 2017.

\bibitem[\protect\citeauthoryear{Lynce and Silva}{2004}]{conf/sat/LynceM04}
In{\^{e}}s Lynce and Jo{\~{a}}o P.~Marques Silva.
\newblock On computing minimum unsatisfiable cores.
\newblock In {\em Proceedings of {SAT}}, 2004.

\bibitem[\protect\citeauthoryear{Marques-Silva \bgroup \em et al.\egroup
  }{2020}]{marques2020reasoning}
Joao Marques-Silva, Carlos Menc{\'\i}a, et~al.
\newblock Reasoning about inconsistent formulas.
\newblock In {\em Proceedings of IJCAI, Survey track}, 2020.

\bibitem[\protect\citeauthoryear{Marques-Silva}{2010}]{marques2010minimal}
Joao Marques-Silva.
\newblock Minimal unsatisfiability: Models, algorithms and applications.
\newblock In {\em 2010 40th IEEE International Symposium on Multiple-Valued
  Logic}, 2010.

\bibitem[\protect\citeauthoryear{Miller \bgroup \em et al.\egroup
  }{2019}]{xai-ijcai}
Tim Miller, Rosina Weber, and Daniele Magazzeni, editors.
\newblock {\em Proceedings of the IJCAI 2019 Workshop on Explainable Artificial
  Intelligence}, 2019.

\bibitem[\protect\citeauthoryear{Miller}{2019}]{miller2019explanation}
Tim Miller.
\newblock Explanation in artificial intelligence: Insights from the social
  sciences.
\newblock {\em Artificial Intelligence}, 267:1--38, 2019.

\bibitem[\protect\citeauthoryear{Oh \bgroup \em et al.\egroup
  }{2004}]{DBLP:conf/dac/OhMASM04}
Yoonna Oh, Maher~N. Mneimneh, Zaher~S. Andraus, Karem~A. Sakallah, and Igor~L.
  Markov.
\newblock {AMUSE:} a minimally-unsatisfiable subformula extractor.
\newblock In {\em Proceedings of {DAC}}, pages 518--523, 2004.

\bibitem[\protect\citeauthoryear{Reiter}{1987}]{ai/Reiter87}
Raymond Reiter.
\newblock A theory of diagnosis from first principles.
\newblock {\em {AIJ}}, 32(1):57--95, 1987.

\bibitem[\protect\citeauthoryear{Rossi \bgroup \em et al.\egroup
  }{2006}]{fai/Rossi06}
Francesca Rossi, Peter van Beek, and Toby Walsh, editors.
\newblock {\em Handbook of Constraint Programming}, volume~2 of {\em
  Foundations of Artificial Intelligence}.
\newblock Elsevier, 2006.

\bibitem[\protect\citeauthoryear{Saikko \bgroup \em et al.\egroup
  }{2016}]{DBLP:conf/kr/SaikkoWJ16}
Paul Saikko, Johannes~Peter Wallner, and Matti J{\"{a}}rvisalo.
\newblock Implicit hitting set algorithms for reasoning beyond {NP}.
\newblock In {\em Proceedings of {KR}}, pages 104--113, 2016.

\bibitem[\protect\citeauthoryear{Sqalli and
  Freuder}{1996}]{sqalli1996inference}
Mohammed~H Sqalli and Eugene~C Freuder.
\newblock Inference-based constraint satisfaction supports explanation.
\newblock In {\em AAAI/IAAI, Vol. 1}, pages 318--325, 1996.

\end{thebibliography}
}

\end{document}